\documentclass{article}

\usepackage{amsmath}
\usepackage{amsfonts}
\usepackage{amsthm}
\usepackage{xcolor}
\usepackage{mathtools}
\usepackage{subcaption}
\usepackage{authblk}
\usepackage[backend=bibtex]{biblatex}
\usepackage{tikz}

\usepackage[]{hyperref}
\hypersetup{
    pdftitle={Invariance Properties of the Natural Gradient in Overparametrised Systems},
    bookmarksnumbered=true,     
    bookmarksopen=true,         
    bookmarksopenlevel=1,       
    colorlinks=true,            
    pdfstartview=Fit,           
    pdfpagemode=UseOutlines,    
}

\usetikzlibrary{decorations.markings, calc, arrows}
\newcommand{\vlen}{0.65}

\DeclareMathOperator{\grad}{grad}
\DeclareMathOperator{\vspan}{span}
\DeclareMathOperator{\Smooth}{Smooth}

\DeclareMathOperator*{\argmin}{arg\,min}

\newtheorem{theorem}{Theorem}
\newtheorem{lemma}{Lemma}
\newtheorem{proposition}{Proposition}

\theoremstyle{definition}
\newtheorem{definition}{Definition}

\theoremstyle{remark}
\newtheorem{remark}{Remark}

\newcommand{\cM}{\mathcal{M}}
\newcommand{\cL}{\mathcal{L}}
\newcommand{\cZ}{\mathcal{Z}}

\newcommand{\bR}{\mathbb{R}}
\newcommand{\parxi}[2][]{\partial_{#2}(\xi_{#1})}

\newcommand{\alternativeParametrisation}{\psi}
\newcommand{\alternativeParameter}{\theta}
\newcommand{\alternativeParameterSpace}{\Theta}
\newcommand{\alternativePartial}{{\partial}}
\newcommand{\alternativeG}{{G}}
\newcommand{\alternativeNabla}{\nabla}

\newcommand{\nablanatpgrad}{\widetilde\nabla}

\newcommand{\der}[2][]{\frac{\partial{#1}}{\partial{#2}}}

\newcommand{\obj}{\mathcal{O}}

\newcommand{\natgrad}{\grad^\cM_p \cL}  
\newcommand{\natgrada}[1]{\grad^\cM_{#1} \cL} 

\newcommand{\natpgrad}{\nablanatpgrad_\xi \cL} 
\newcommand{\natpmgrad}{d\phi_\xi\nablanatpgrad_\xi \cL}
\newcommand{\natpmgrada}[1]{d\phi_{#1}\nablanatpgrad_{#1}\cL} 
\newcommand{\natapmgrada}[1]{d\alternativeParametrisation_{{#1}}\nablanatpgrad_{{#1}}\cL}

\newcommand{\natapgrad}{\nablanatpgrad_\alternativeParameter \cL}
\newcommand{\natapmgrad}{d\alternativeParametrisation_\alternativeParameter\nablanatpgrad_\alternativeParameter \cL}

\title{Invariance Properties of the Natural Gradient in Overparametrised Systems\footnote{A previous version of this paper \cite{oostrum2021parametrisation} was presented at the International Conference on Geometric Science of Information 2021 in Paris. Here  invariance properties for proper parametrisations were discussed. The current paper treats invariance for both proper and non-proper parametrisations.}}
\author[1]{Jesse van Oostrum}
\author[2]{Johannes Müller}
\author[1,3,4]{Nihat Ay}
\affil[1]{Hamburg University of Technology, Hamburg, Germany}
\affil[2]{Max Planck Institute for Mathematics in the Sciences, Leipzig, Germany}
\affil[3]{Leipzig University, Leipzig, Germany}
\affil[4]{Santa Fe Institute, Santa Fe, USA}

\date{}

\begin{document}

\maketitle
 
\begin{abstract}
    The natural gradient field is a vector field that lives on a model equipped with a distinguished Riemannian metric, e.g. the Fisher-Rao metric, and represents the direction of steepest ascent of an objective function on the model with respect to this metric. In practice, one tries to obtain the corresponding direction on the parameter space by multiplying the ordinary gradient by the inverse of the Gram matrix associated with the metric. We refer to this vector on the parameter space as the natural parameter gradient. In this paper we study when the pushforward of the natural parameter gradient is equal to the natural gradient. Furthermore we investigate the invariance properties of the natural parameter gradient. Both questions are addressed in an overparametrised setting. 
\end{abstract}

\section{Introduction}
Within the field of deep learning, gradient methods have become ubiquitous tools for parameter optimisation. Standard gradient optimisation procedures use the vector of coordinate derivatives of the objective function as the update direction of the parameters. This is implicitly assuming a Euclidean geometry on the space of parameters. It can be argued that this is not always the most natural choice of geometry. Instead one can choose a more natural geometry for the problem at hand and then determine the Riemannian gradient of the objective function for this natural geometry, resulting in the so called \emph{natural gradient}. The \emph{natural gradient method} is the optimisation algorithm that performs discrete parameter updates in the direction of the natural gradient. This method was first proposed by Amari \cite{amari1998} using the geometry induced by the Fisher-Rao metric. It is an active field of study within information geometry \cite{ay2020,lin2021tractable,grosse2016kronecker} and has been shown extremely effective in many applications \cite{ay2013selection,van2012reinforcement,varady2020}. More recently, also other geometries on the model have been studied, such as the Wasserstein geometry \cite{malago2018wasserstein, li2018natural}. The natural gradient is defined independently of a specific parametrisation.  Although it is an open problem, there is work supporting the idea that the efficiency of learning of the method is due to this invariance \cite{zhang2019fast}. 

In practice, the update direction of the parameters is given by the ordinary gradient multiplied by the inverse of the Gram matrix associated with the metric on the model. We will refer to this vector on the parameter space as the \emph{natural parameter gradient}. In order to determine whether this direction is desired we have to map this vector to the model, since it is the location on the model, not on the parameter space, that determines the performance of the model. In non-overparametrised systems it can be shown that in a non-singular point on the model, the pushforward of the natural parameter gradient is equal to the natural gradient. Furthermore the natural parameter gradient can be called parametrisation invariant in this case \cite{martens2020new}. In many practical applications of machine learning, and in particular deep learning, one deals however with overparametrised models, in which different directions on the parameter space correspond to a single direction on the model. In this case, the Gram matrix is degenerate and we use a generalised inverse to calculate the natural parameter gradient. In this paper, we will investigate whether the pushforward of the natural parameter gradient remains equal to the natural gradient in the overparametrised setting. 

The Moore-Penrose (MP) inverse is the canonical choice of generalised inverse for the natural parameter gradient \cite{bernacchia2019exact}. The definition of the MP inverse is based on the Euclidean inner product defined on the parameter space. Using the MP inverse is therefore thought to affect the parametrisation invariance of the natural parameter gradient  \cite{ollivier2015}, and thus potentially the performance of the natural gradient method. In this paper we propose two different notions of invariance. The first evaluates the invariance of the natural parameter gradient by examining the behaviour of its pushforward on the model. The second looks at the behaviour on the parameter space itself. Since the location and direction on the model is what matters, we argue that the former is of greater importance. 

\section{The natural gradient} \label{main}

Let $(\cZ, g)$ be a Riemannian manifold, $\Xi$ be a parameter space that we assume to be an open subset of $\bR^d$, $\phi\colon \Xi \to \cZ$ a smooth map (taking the role of a parametrisation\footnote{Within the context of differential geometry, a parametrisation usually means a local diffeomorphism onto its image. This is also referred to as a coordinate system. In our context, we use a more general definition of parametrisation, where $\phi\colon \Xi \to \cZ$ no longer needs to be a diffeomorphism onto its image but only smooth.}), $\mathcal{M} \coloneqq \phi(\Xi) \subset \cZ$ a model\footnote{Note that in the literature this is also called a parametrised model, statistical manifold, or in the context of machine learning, a neuromanifold. We choose the word model, however, to emphasize that we do not assume a manifold structure on $\cM$.}, and $\cL\colon \cZ \to \bR$ a smooth (objective) function (see Figure \ref{param1a}). We call $p \in \mathcal{M}$ \textit{non-singular} if $\mathcal{M}$ is locally an embedded submanifold of $\cZ$ around $p$ and we denote the set of non-singular points with $\Smooth(\mathcal{M})$. A point $p$ is called \textit{singular} if it is not non-singular. The Riemannian gradient of $\cL$ on $\cZ$ is defined implicitly as follows:  
\begin{equation} \label{grad0}
    g_p\left(\grad_p^\cZ \mathcal{L}, \cdot\right) = d\mathcal{L}_p(\cdot).
\end{equation}
By the Riesz representation theorem, this defines the gradient uniquely. 
\begin{definition}[Natural gradient]
    For $p \in \Smooth(\cM)$ the Riemannian gradient of  $\mathcal{L}|_\cM$ on the model $\cM$ is called the \emph{natural gradient} and is denoted $\natgrad$.
\end{definition}
\noindent It is easy to show that:
\begin{equation}
    \natgrad  = \Pi_p (\grad^\cZ_p \cL), 
\end{equation}
where $\Pi_p$ is the projection onto $T_p\cM$. We define the pushforward of the tangent vector on the parameter space through the parametrisation as $\parxi{i} \coloneqq d\phi_\xi \left(\left.\frac{\partial}{\partial \xi^i}\right|_\xi\right)$, and the Gram matrix $G(\xi)$ by $G_{ij}(\xi) \coloneqq g_{\phi(\xi)}\left(\parxi{i}, \parxi{j}\right)$. We denote the vector of coordinate derivatives with $\nabla_\xi \cL \coloneqq \left(\partial_1(\xi) \cL, ..., \partial_d(\xi)\cL\right) = \left(\frac{\partial \cL \circ \phi}{\partial \xi^1}(\xi),..., \frac{\partial \cL \circ \phi}{\partial \xi^d}(\xi)\right) \in \bR^d$. Let $\xi$ be such that $\phi(\xi) \in \Smooth(\mathcal{M})$. We say that a parametrisation is \textit{proper} in $\xi$ when: $\vspan\left(\left\{\parxi{1}, ..., \parxi{d}\right\}\right) = T_{\phi(\xi)}\mathcal{M}$.  Furthermore, following the Einstein summation convention, we write $a^i b_i$ for the sum $\sum_i a^i b_i$. 

\tikzstyle{dot} = [fill, circle, inner sep=0, minimum size=1.4mm]
\tikzstyle{gridline} = [gray!50, very thin]
  \begin{figure}
  
    \tikzset{ 
    seagull/.pic={
    \draw [gridline] (-.8,.5) to [bend left=3] (.9,.5);
    \draw [gridline] (-1,0) to [bend left=3] (.7,0);
    \draw [gridline] (-1.2,-.5) to [bend left=3] (.5,-.5);
    \draw [gridline] (-1., -.7) to [bend left=10](-.45, .7);
    \draw [gridline] (-.4, -.7) to [bend left=10](0.15, .7);
    \draw [gridline] (0.2, -.7) to [bend left=10](0.75, .7);
    }
}
    \begin{tikzpicture}[auto]
    \node (xi) at (0, 0) [dot, label=below left:$\xi$] {};
    \node (pp) at (6,0) {}; 
    \node (R) at (10, 0) {$\mathbb{R}$};
    \node (p) at ($(pp)+ (-.3,.05)$) [dot, label=below left:$p$] {}; 
    \node at ($(xi) + (1.4,1.2)$) {$\Xi$};
    \node at ($(p)+(1.8,1.2)$) {$\mathcal{Z}$};
    \node at ($(p) +(1.8, .2)$) {$\mathcal{M}$};
   
    \draw [step=.9cm, gridline] ($(xi)+(-1.2,-1.2)$) grid ($(xi)+(1.2,1.2)$);
    \pic [scale=1.5] at (pp) {seagull};
    
    \draw [bend right, thick, gray] (p) to  +(-1.5,0);
        \draw [bend right, thick, gray] (p) to  +(1.5,0);

    \draw [->] (xi) -- +(0,.7) node [anchor=south west, inner sep = 0, node font=\small] {$\der{\xi^2}|_\xi$};
    \draw [->] (xi) -- +(0.7,0) node [circle, anchor=north west, inner sep=0, node font=\small] {$\der{\xi^1}|_\xi$};
    \draw [->] (p) -- ($(p)+(330:1.1)$) node [anchor=north west, inner sep = 1, node font=\small] {$\partial_1(\xi)$};
    \draw [->] (p) -- ($(p)+(150:.9)$) node [anchor=south, inner sep = 3, node font=\small] {$\partial_2(\xi)$};

    \draw [->, black!80, very thin, bend left=0] ($(xi) + (1.7,0)$) to node {$\phi$} ($(pp) + (-2.5,0)$);
    \draw [->, black!80, very thin, bend left=0] ($(pp) + (2,0)$) to node {$\mathcal{L}$} ($(R) + (-.5,0)$);

    \node at (0, 0) [dot] {};
    \node at ($(pp)+ (-.3,.05)$) [dot] {}; 

    \end{tikzpicture}
    \caption{Parametrisation and objective function}
    \label{param1a}
    \end{figure}
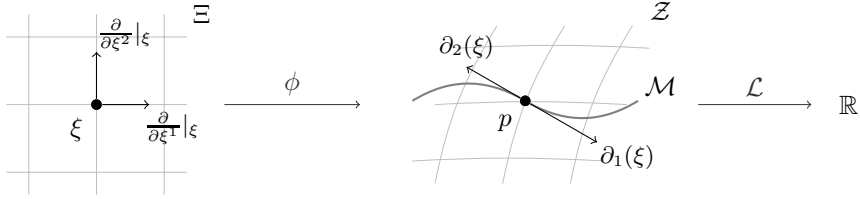

\begin{definition}[Generalised inverse]
    A \emph{generalised inverse} of an $n \times m$ matrix $A$, denoted $A^+$, is an $m \times n$ matrix satisfying the following property:
    \begin{equation}
        A A^+ A = A.
    \end{equation} 
\end{definition}
\noindent Note that this definition implies that for $w \in \bR^n$ in the image of A, i.e. $w = Av$ for some $v \in \bR^m$, we have: 
\begin{equation}
    AA^+ w = AA^+ A v = Av = w.
\end{equation}
This shows that $A A^+$ is the identity operator on the image of $A$.
\begin{definition}[Natural parameter gradient]
    We define the \emph{natural parameter gradient} to be the following vector on the parameter space:
    \begin{equation} \label{natpgrad}
         \natpgrad \coloneqq  \left(G^+(\xi) \nabla_\xi \cL \right)^i \left.\der{\xi^i}\right|_\xi. 
    \end{equation}
    The pushforward of this vector, given by:  
    \begin{equation} \label{natpmgrad}
        \natpmgrad = \left(G^+(\xi) \nabla_\xi \cL \right)^i \partial_i,
   \end{equation}
    is called the \emph{natural parameter gradient on $\cM$}. 
\end{definition}
\noindent Often, the natural parameter gradient is denoted with matrix notation as follows:
\begin{equation} 
    \natpgrad \coloneqq  G^+(\xi) \nabla_\xi \cL,
\end{equation}
where an identification between the canonical basis of $\bR^d$ and the vectors $\left.\der{\xi^i}\right|_\xi$ is made implicitly. \\

We are now in the position to state the main result of the paper:
\begin{theorem} \label{theorem1}
    Let $\xi \in \Xi$ and $p = \phi(\xi) \in \cM$. We have:
    \begin{equation} \label{thm1b}
        \natpmgrad = \Pi_\xi \left(\grad_p^\cZ \cL\right) ,
    \end{equation}
    where $\Pi_\xi$ is the projection onto $\vspan\{\partial_i(\xi)\}_i$. In particular, when $\phi(\xi)$ is non-singular and  $\vspan\{\partial_i(\xi)\}_i = T_p\cM$ we have:
     \begin{equation} \label{thm1}
        \natpmgrad = \natgrad.
    \end{equation}
\end{theorem}
\noindent This theorem implies that under certain conditions the pushforward of the natural parameter gradient is equal to the natural gradient. Furthermore we see that in general the natural parameter gradient on $\cM$ is dependent on the choice of parametrisation through $\Pi_\xi$, but becomes invariant when the coordinate vectors span the full tangent space of $\cM$. In the next section we will study in more detail the invariance properties of the natural parameter gradient.\\

\noindent The proof of Theorem \ref{theorem1} will be based on the following result from linear algebra: 
\begin{lemma}
    Let $(V, \langle \cdot, \cdot \rangle)$ be a finite-dimensional inner product space and $V^*$ its dual space. Let $\{e_i\}_{i \in \{1,...,d\}} \subset V$ (not necessarily linearly independent), $G$ the matrix defined by $G_{ij} = \langle e_i, e_j \rangle$, $\omega \in V^*$, $v$ such that $\langle v, \cdot \rangle = \omega(\cdot)$ and $\Pi$ the projection on the space $\vspan\{e_i\}_i$. Then, 
    \begin{equation} \label{projv}
        \Pi(v) = \left(G^+\right)^{ij}\omega(e_j)e_i.
    \end{equation}
\end{lemma}
\begin{proof}
    Start by noting $\Pi(v)$ is uniquely defined by the fact that $\langle \Pi(v), w \rangle = \omega(w)$, for $w \in \vspan\{e_i\}_i$ and $\langle \Pi(v), w \rangle = 0$ for $w \in \left(\vspan\{e_i\}_i\right)^\perp$. Since the RHS of \eqref{projv} lies in the span of $\{e_i\}_i$, it remains to show that for an arbitrary vector $w = w^i e_i \in \vspan\{e_i\}_i$ we have:
    \begin{equation}
        \langle \left(G^+\right)^{ij}\omega(e_j)e_i, w^k e_k \rangle = \omega(w).
    \end{equation}
    Working out the LHS gives:
    \begin{align}
        \langle \left(G^+\right)^{ij}\omega(e_j)e_i, w^k e_k \rangle &= G_{ik} \left(G^+\right)^{ij}\omega(e_j) w^k \\
        &= w^k G_{ki}\left(G^+\right)^{ij}G_{jl} v^l\\
        &= w^k G_{kl} v^l \\
        &= w^k \omega(e_k) \\
        &= \omega(w),
    \end{align}
    where we use the fact that: $\omega(e_i) = G_{ij}v^j$ and the symmetry of $G$ in the second equality.
    
\end{proof}

\begin{proof}[Proof of Theorem \ref{theorem1}]
We now let $T_p\cM$ take the role of $V$, $d\cL_p$ the role of $\omega$, $\parxi{i}$ the role of $e_i$, and $\grad_p \cL$ the role of $v$. Equation \eqref{thm1b} now follows immediately. When the tangent vectors $\{\partial_i(\xi)\}_i$ span the whole tangent space of $\cM$ at $p$, $\Pi_\xi$ becomes the identity on $T_p\cM$. This gives Equation \eqref{thm1}.
\end{proof}

\section{Invariance properties of the natural parameter gradient} \label{paraminv}

In this section we study the invariance properties of the natural parameter gradient by using an alternative parametrisation of $\cM$ given by: 
\begin{equation}
    \alternativeParametrisation\colon \alternativeParameterSpace \ni \alternativeParameter \mapsto \alternativeParametrisation(\alternativeParameter) \in \cZ.
\end{equation} 
Note that $G^+(\xi), \nabla_{\xi} \cL$ and $\parxi{i}$ in the definition of $\natpmgrad$ all implicitly depend on the parametrisation $\phi$. For an alternative parametrisation $\alternativeParametrisation$ we will therefore write: $\alternativePartial_i(\alternativeParameter) \coloneqq d\alternativeParametrisation_{\alternativeParameter} \left( \der{\alternativeParameter}|_{\alternativeParameter}\right)$, $\alternativeG_{ij}(\alternativeParameter) \coloneqq g_{\alternativeParametrisation(\alternativeParameter)} \left(\alternativePartial_i(\alternativeParameter), \alternativePartial_j(\alternativeParameter)\right)$, and $\alternativeNabla_{\alternativeParameter} \cL \coloneqq (\alternativePartial_1(\alternativeParameter) \cL, ..., \alternativePartial_d(\alternativeParameter)\cL)$. 

The invariance properties can be studied from the perspective of the model and from the perspective of the parameter space itself. Since the former is of more importance, we will start with this one.

\begin{figure} [ht]
    \centering
    \tikzstyle{dot} = [fill, circle, inner sep=0, minimum size=1.4mm]
    \tikzstyle{grid} = [gray!50, very thin, step=.5cm]
    
    \begin{tikzpicture} [auto, scale=1.2]
    \node (eta) at (-6, 0) [dot, label=below left:$\alternativeParameter$] {};
    \node (xi) at (0,0) [dot, label=below left:$\xi$]  {}; 
    \node (p) at (0,3.5) [dot, label=below left:$p$]  {}; 
    
    \draw [grid] ($(eta)+(-1,-1)$) grid ($(eta)+(1.001,1)$);
    \draw [grid] ($(xi)+(-1,-1)$) grid ($(xi)+(1,1)$); 
    \draw [bend right, thick, gray] (p) to  +(-1.5,0);
    \draw [bend right, thick, gray] (p) to  +(1.5,0);
    \node at ($(eta) + (1.4,1.2)$) {$\alternativeParameterSpace$};
    \node at ($(xi) + (1.4,1.2)$) {$\Xi$};
    \node at ($(p)+(2,0)$) {$\mathcal{M}$};
    
    \draw [->, black!80, very thin, bend right=15] ($(eta) + (1.5,0)$) to node {$f$} ($(xi) + (-1.5,0)$);
    \draw [->, black!80, very thin, bend left=30] ($(eta) + (0,1.5)$) to node {$\alternativeParametrisation$} ($(p) + (-2.2,0)$);
    \draw [->, black!80, very thin, bend right=0 ] ($(xi) + (0,1.5)$) to node [swap] {$\phi$} ($(p) + (0,-1)$); 
    
    \draw [->] (eta) -- +(0,.7) node [anchor=south west, inner sep = 0, node font=\small] {$\der{\alternativeParameter^2}|_{\alternativeParameter}$};
    \draw [->] (eta) -- +(0.7,0) node [circle, anchor=north west, inner sep=0, node font=\small] {$\der{\alternativeParameter^1}|_{\alternativeParameter}$};
    \draw [->] (xi) -- +(0,.7) node [anchor=south west, inner sep = 0, node font=\small] {$\der{\xi^2}|_\xi$};
    \draw [->] (xi) -- +(0.7,0) node [circle, anchor=north west, inner sep=0, node font=\small] {$\der{\xi^1}|_\xi$};
    \draw [->] (p) -- ($(p)+(330:1.1)$) node [anchor=north west, inner sep = 1, node font=\small] {$\partial_i(\xi)$};
    \draw [->] (p) -- ($(p)+(150:.9)$) node [anchor=south, inner sep = 3, node font=\small] {$\alternativePartial_i(\alternativeParameter)$};

    \node at (-6, 0) [dot] {};
    \node at (0,0) [dot]  {};

    
    \end{tikzpicture}  
    \caption{Two parametrisations of $\cM$}
    \label{twoparamsa}
  \end{figure}
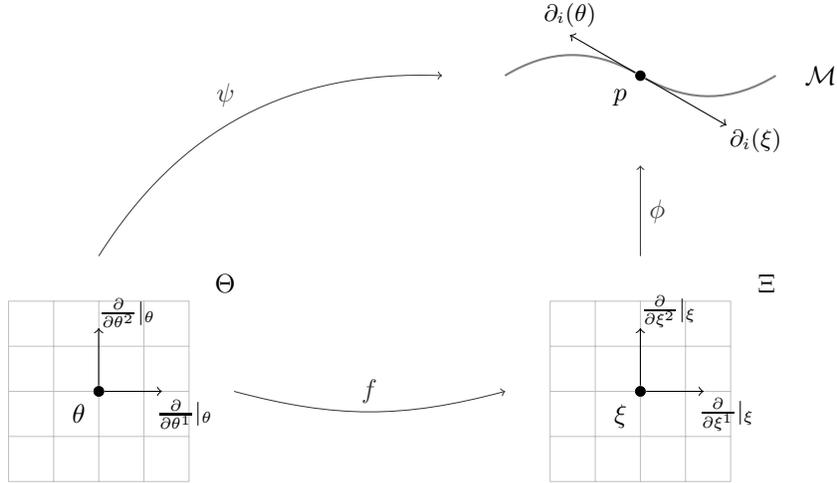

\subsection{Parametrisation dependence and reparametrisation invariance on the model} \label{invman}
 A parametrisation can be used to represent tangent vectors on the model space by elements of $\bR^d$. A representation (of vectors on $\cM$) can be interpreted as the map $\obj \colon (\phi, \xi) \mapsto \obj(\phi, \xi) \in T_\xi\Xi \ (\cong \bR^d)$ that takes a parametrisation-coordinate pair and assigns a tangent vector on the parameter space to it. The natural parameter gradient defined by $\natpgrad = \left(G^+(\xi) \nabla_\xi \cL \right)^i \left.\der{\xi^i}\right|_\xi$ in Equation \eqref{natpgrad} is an example of a representation, where the dependence on $\phi$ on the RHS is implicit. Naively, one could define invariance of a representation in the following way:
 
\begin{definition}[Parametrisation independence]\label{parind}
    Let $\cM$ be a model. A representation $\obj(\cdot, \cdot)$ is called \emph{parametrisation independent} if for any pair $\phi, \alternativeParametrisation$ of parametrisations of $\cM$, and coordinates $\xi,\alternativeParameter$ such that $\alternativeParametrisation(\alternativeParameter) = \phi(\xi)$, the following holds: 
    \begin{equation}\label{inveq}
        d\alternativeParametrisation_{\alternativeParameter} \obj(\alternativeParametrisation,\alternativeParameter) = d\phi_\xi \obj(\phi,\xi).
    \end{equation}
\end{definition}
It turns out that this is not a very useful definition. As we will see, no non-trivial representation can be parametrisation independent in the sense of this definition. We will illustrate this in Example 1 and 2 below for the natural parameter gradient on specific models. A formal proof can be found in the Appendix \ref{def4lim}.\\ 
In order to overcome the limitation of Definition 4, we propose the following more suitable definition of invariance of a representation:
\begin{definition}[Reparametrisation invariance]\label{invrep}
    Let $\cM$ be a model. A representation $\obj(\cdot, \cdot)$ is called \emph{reparametrisation invariant} if for any pair $\phi, \alternativeParametrisation$ of parametrisations of $\cM$, such that $\alternativeParametrisation = \phi \circ f$ for a diffeomorphism $f\colon \alternativeParameterSpace \to \Xi$, and coordinates $\xi,\alternativeParameter$ such that $\alternativeParameter = f^{-1}(\xi)$, the equality \eqref{inveq} holds.
\end{definition}
Due to the extra requirement of the existence of the reparametrisation function $f$ in Definition \ref{invrep}, we get the following central result of this paper:  

\begin{theorem} \label{cor1}
     The natural parameter gradient is reparametrisation invariant. 
\end{theorem}

\begin{proof}  
    By Definition \ref{invrep}, we need to show that for $\alternativeParametrisation = \phi \circ f$ and $\alternativeParameter = f^{-1}(\xi)$ we have:
    \begin{equation}
        d\alternativeParametrisation_{\alternativeParameter}\natapgrad = d\phi_\xi \natpgrad.
    \end{equation}
    Since the differential $df_{\alternativeParameter}$ is surjective, we have that $\vspan\{\partial_i(\xi)\}_i = \vspan\{\alternativePartial_j(\alternativeParameter)\}_j$. Therefore, by using equation Equation \eqref{thm1b} of Theorem \ref{theorem1}, we get:
    \begin{equation}
        d\alternativeParametrisation_{\alternativeParameter}\natapgrad = \Pi_\theta \left(\grad_{\psi(\theta)}^\cZ \cL\right) =  \Pi_\xi \left(\grad_{\phi(\xi)}^\cZ \cL\right) = d\phi_\xi \natpgrad,
    \end{equation}
    which is what we wanted to show. 
\end{proof}
\begin{remark} \label{remark:diffeomorphismAssumption}
    Note that under the extra assumptions that 
    $\cM$  is a smooth manifold and all parametrisations are required to be diffeomorphisms Definitions \ref{parind} and \ref{invrep} are equivalent, since the composition $f = \alternativeParametrisation \circ \phi^{-1}$ is a diffeomorphism. These assumptions are often implicitly made when referring to the invariance of the natural gradient. However, as we will see below, this is no longer the case in our more general setting. 
\end{remark}

\subsubsection*{Example 1}

This example will be of a graphical nature. Consider the parametrisation $\phi$ that is the composition of the 2 maps in Figure \ref{param5b}. Now let $\alternativeParametrisation\colon \Xi \to \cM$ be this parametrisation but with a 90 degree rotation around $\phi(\xi)$ applied before projecting down to $\cM$. 

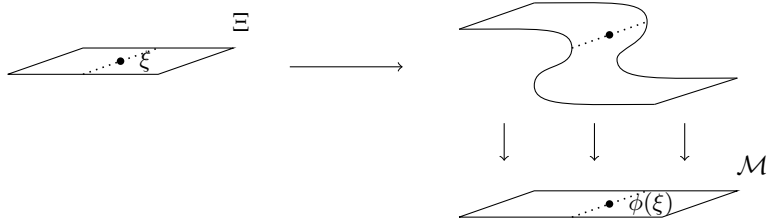
\begin{figure}[ht]
\centering
\tikzstyle{sdot} = [fill, circle, inner sep=0, minimum size=1mm]
\begin{tikzpicture}[scale = .5]


\node (a) at (0,0) {};
\node (b) at ($(a) + (3,-.5)$) {};
\node (c) at ($(b) + (-1, -1)$) {};
\node (d) at ($(c) + (3.2, -.5)$) {};

\draw (a.center) .. controls ($(a) + (2,0)$) and ($(b) + (0,.6)$) ..  (b.center) .. controls ($(b) + (0,-.6)$) and ($(c) + (0,.6)$) .. (c.center) .. controls ($(c) + (0,-.6)$) and ($(d) + (-2,0)$) .. (d.center);

\node (aa) at ($(a) + (2,.7)$) {};
\node (bb) at ($(aa) + (3,-.5)$) {};
\node (cc) at ($(bb) + (-.8, -1)$) {};
\node (dd) at ($(cc) + (3.2, -.5)$) {};

\draw (aa.center) .. controls ($(aa) + (2,0)$) and ($(bb) + (0,.6)$) ..  (bb.center) .. controls ($(bb) + (0,-.6)$) and ($(cc) + (0,.3)$) .. (cc.center) .. controls ($(cc) + (0,-.6)$) and ($(dd) + (-2,0)$) .. (dd.center);

\draw (a.center) to (aa.center);
\draw (d.center) to (dd.center);
\draw [dotted, semithick] (b.center) to (bb.center);

\node [sdot] at ($(b) + (1, .35)$) {};

\node (x) at (-12, -1.2) {};
\node at ($(x) + (6.2, 1.4)$) {$\Xi$};

\draw (x.center) to ($(x) + (4,0)$) to ($(x) + (6 ,.7)$) to ($(x) + (2,.7)$) to (x.center);

\draw [dotted, semithick] ($(x) + (2,0)$) to ($(x) + (4,.7)$);

\node [sdot,label={[label distance = .7mm, font=\small] right:$\xi$}] at ($(x) + (3, .35)$) {};

\node (y) at (0, -5) {};
\node at ($(y) + (7.8, 1.4)$) {$\cM$};

\draw (y.center) to ($(y) + (5.4,0)$) to ($(y) + (7.4,.7)$) to ($(y) + (2,.7)$) to (y.center);

\draw [dotted, semithick] ($(y) + (3,0)$) to ($(y) + (5,.7)$);
\node [sdot,label={[label distance = .7mm, font=\small] right:$\phi(\xi)$}] at ($(y) + (4, .35)$) {};

\draw[->] ($(b) + (-7.5,-.5)$) to ($(b) + (-4.5,-.5)$);
\draw[->] (1.2, -2.5) to (1.2, -3.5);
\draw[->] (3.6, -2.5) to (3.6, -3.5);
\draw[->] (6, -2.5) to (6, -3.5);

\end{tikzpicture}
\caption{Parametrisation with a non-surjective span of the parameter vectors}
\label{param5b}
\end{figure}

Note that the spans of the parameter vectors have trivial intersection as depicted in Figure \ref{param5c}. This immediately implies that for any representation we have: $ d\alternativeParametrisation_{\alternativeParameter} \obj(\alternativeParametrisation,\alternativeParameter) \neq d\phi_\xi \obj(\phi,\xi)$ except when both sides are equal to zero. In particular, we can let the natural gradient be as in Figure \ref{param5d}. We know from Theorem \ref{theorem1} that the natural parameter gradients on $\cM$ will be the projections of the natural gradient onto the respective spans of the parameter vectors as depicted in the figure. Note that the projection should be orthogonal with respect to the inner product $g_p$, which we have chosen here to be Euclidean for ease of illustration.

\begin{figure}[ht]
    \begin{subfigure}{0.5\textwidth}
    \centering
    \begin{tikzpicture}[scale=1.2]
    \node at (2.3, 1.8) {$\cM$};
    \draw [very thin]  (-2,-1.5) rectangle (2,1.5);
    \node[dot, label=below left: $\phi(\xi)$] at (0,0) {};
    \draw[dotted, semithick] (-2,0) to (2,0);
    \draw[dotted, semithick] (0, -1.5) to (0,1.5);
    \draw[->] (0,0) -- (0,.6) node [anchor=south west, node font=\small] {$\vspan\{d\phi_\xi(\der{\xi^i}|_\xi)\}_i$};
    \draw[->] (0,0) -- (0.6,0) node [anchor=155, node font=\small] {$\vspan\{d\psi_\xi(\der{\xi^j}|_\xi)\}_j$};
    \end{tikzpicture}
    \caption{Spans of both sets of parameter vectors}
    \label{param5c}
    \end{subfigure}
    \begin{subfigure}{0.5\textwidth}
    \centering
    \begin{tikzpicture}[scale=1.2]
    \node at (2.3, 1.8) {$\cM$};
    \draw [very thin]  (-2,-1.5) rectangle (2,1.5);
    \node[dot, label=below left: $\phi(\xi)$] at (0,0) {};
    \draw[dotted, semithick] (-2,0) to (2,0);
    \draw[dotted, semithick] (0, -1.5) to (0,1.5);
    \draw[->] (0,0) -- (1, .8) node [anchor= west, node font=\small] {$\natgrada{\phi(\xi)}$};
    \draw[->] (0,0) -- (0,.8) node [anchor=205, node font=\small, outer sep = 1.8] {$\natpmgrad$};
    \draw[->] (0,0) -- (1,0) node [anchor=140, node font=\small] {$\natapmgrada{\xi}$};
    \draw [dotted, thin] (0,.8) -- (1,.8) -- (1,0);
    \end{tikzpicture}
    \caption{Projected gradient vectors}
    \label{param5d}
    \end{subfigure}
    \caption{Spans and gradient vectors on the model of both parametrisations}
    \end{figure}
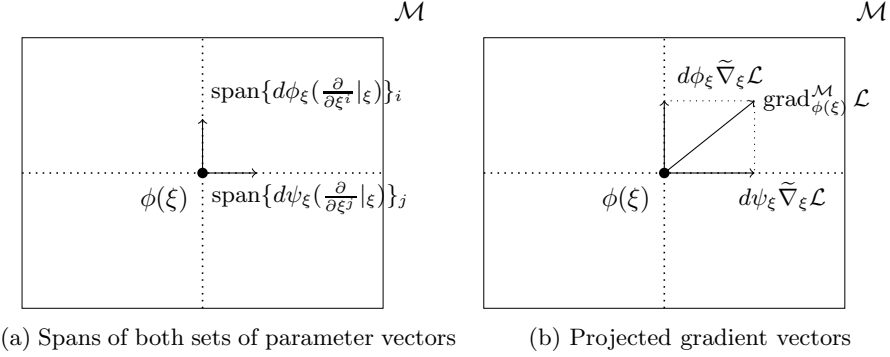
    
\noindent This example shows that for non-singular points, we can construct two parametrisations that give different natural parameter gradients on the same point of the model. Note that this is not in violation of Theorem \ref{cor1} since there does not exist a diffeomorphism $f$ such that $\alternativeParametrisation = \phi \circ f$.

\subsubsection*{Example 2}

\begin{figure}[ht]
\tikzstyle{dot} = [radius=.7mm]
\tikzstyle{vspace} = [line width = 3, #1!30, >=stealth, <->]
\begin{tikzpicture}[auto, scale=.95]
\draw [gray!50, very thin, step=1.01cm] (-2,-2) grid (2,2);
\node (a) at (-10,0) [label=below:] {};
\node (b) at (-5,0) [label=below:] {};
\node at ($(b)+(-.3,1)$) {$\Xi$};
\node at (1.5,2) {$\mathbb{R}^2$};
\node at (1.5, -.5) {$\mathcal{M}$};
\draw [vspace=blue] (-\vlen,-\vlen) -- (\vlen, \vlen);
\draw [vspace=red] (\vlen,-\vlen) -- (-\vlen, \vlen);
\begin{scope}[decoration={
    markings,
    mark=at position 0.5 with {\arrow[line width=1.5pt]{>}}}
    ] 
    \draw [postaction={decorate}] (a.center) -- (b.center);
    \draw [postaction={decorate}] [rounded corners=6.5mm] (-1,-1) -- (1,1) -- (0,2) -- (-1,1) -- (1,-1);
\end{scope}
\draw [->] (b.center) ++(0,2) to [bend left=20] node {$\phi$} (-2,2);
\fill (0,0) circle [dot] node [label=right:$p$]{};
\fill ($(a.center)!0.25!(b.center)$) circle [dot] node [label=below:$\xi_1$] {};
\fill ($(a.center)!0.75!(b.center)$) circle [dot] node [label=below:$\xi_2$] {};
\end{tikzpicture}
\caption{Parametrisation that contains a singular point}
\label{examparam}
\end{figure}
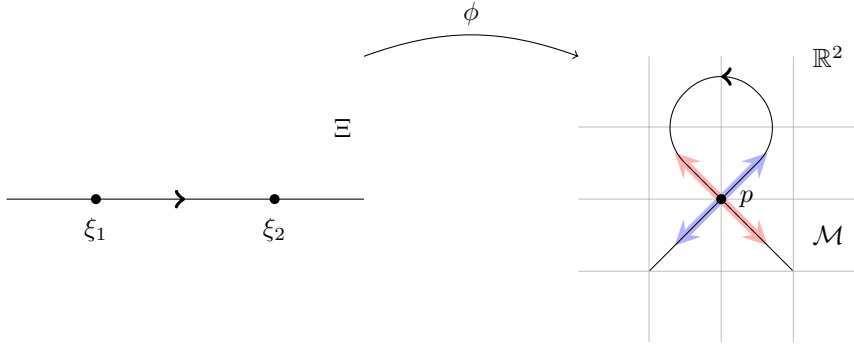

Let us consider the case in which $\phi$ is a smooth map from an interval on the real line to $\mathbb{R}^2$ as depicted in Figure \ref{examparam}. We have that $\xi_1$ and $\xi_2$ are both mapped to the same point $p$ in $\mathbb{R}^2$. Note that $\cM$ is in this case not a locally embedded submanifold around $p$ and thus $p$ is a singular point. Note that $G{(\xi_1})$ is a real number different from zero and therefore non-degenerate. Calculating the natural parameter gradient on $\cM$ for $\xi= \xi_1$ gives:
\begin{align} 
    \natpmgrada{\xi_1} &= G^+(\xi_1) \nabla_{\xi} \cL \ \parxi[1]{}\\
    &= G^{-1}(\xi_1) \frac{\partial \cL \circ \phi}{\partial \xi}(\xi_1) \parxi[1]{}. \label{grad1}
\end{align}
Since $G^{-1}(\xi_1) \frac{\partial \cL \circ \phi}{\partial \xi}(\xi_1)$ is  a scalar, the resulting vector will lie in the span of $\parxi[1]{}$ illustrated by the blue arrows in the figure.\\

Now let $f\colon\Theta \to \Xi$ be a diffeomorphism such that $f(\theta_1) = \xi_2$. An alternative parametrisation of $\cM$ is given by: 
\begin{equation}
    \alternativeParametrisation = \phi \circ f.
\end{equation}
Calculating the natural parameter gradient at $\theta_1$ for this parametrisation gives:
\begin{equation}
    \natapmgrada{\theta_1} = \alternativeG^{-1}({\theta_1}) \frac{\partial \cL \circ \alternativeParametrisation}{\partial \theta}(\theta_1) \partial(\theta_1).
\end{equation}
Note that this vector is in the span of $\partial(\theta_1)$ denoted by the red arrows in the figure and therefore in general different from \eqref{grad1}. This shows that when $\vspan\{\partial_i(\xi)\}_i \neq \vspan\{\alternativePartial_j(\alternativeParameter)\}_j$ the outcome of $\left(G^+(\xi) \nabla_{\xi} \cL \right)^i \parxi{i}$ can be dependent on the choice of parametrisation and therefore the natural parameter gradient is not parametrisation independent. Note however that this result is not in contradiction with Theorem \ref{cor1} since we do not have $\alternativeParameter_1 = f^{-1}(\xi_1)$. See Appendix \ref{exampleCalculationModel} for a worked-out example of this. \\

\subsection{Reparametrisation (in)variance on the parameter space} \label{paraminvparamspacessec}

In the previous section we have looked at the invariance properties of the natural parameter gradient from the perspective of the  model. One can also study the invariance properties from the perspective of the parameter space as is done for example in Section 12 of \cite{martens2020new}. Translating the definition of invariance given there to our notation gives the following:
\begin{definition}[Reparametrisation invariance on the parameter space]
    A representation $\obj(\cdot, \cdot)$ is called \emph{reparametrisation invariant on the parameter space} if for any pair of parametrisations $\phi, \alternativeParametrisation$ such that $\alternativeParametrisation = \phi \circ f$ for a diffeomorphism $f\colon \alternativeParameterSpace \to \Xi$, and coordinates $\xi,\alternativeParameter$ such that $\alternativeParameter = f^{-1}(\xi)$, we have:
    \begin{equation}\label{paraminvparamspace}
        df_{\alternativeParameter}\obj(\alternativeParametrisation,\alternativeParameter) = \obj(\phi,\xi).
    \end{equation}
\end{definition}
Note that reparametrisation invariance on the parameter space implies re\-parametrisation invariance on the model as defined in Definition \ref{invrep}. Furthermore, it can be shown that when $\cM$ is a smooth manifold and all parametrisations are required to be diffeomorphisms, like in Remark \ref{remark:diffeomorphismAssumption}, this definition is equivalent to Definitions \ref{parind} and \ref{invrep}. In that case, the natural parameter gradient satisfies Equation \eqref{paraminvparamspace}. As we will see below, this is not true for general $\phi$. We would like to argue, however, that this is not a suitable definition of invariance, since multiple vectors on the parameter space can be mapped to the same vector on the model. Therefore inequality on the parameter space does not have to imply inequality on the model.

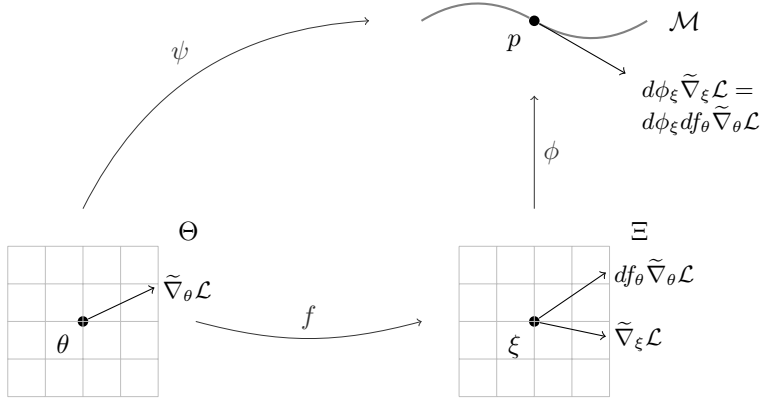
\begin{figure} [ht]
    \centering
    \scalebox{1}{
    \tikzstyle{dot} = [fill, circle, inner sep=0, minimum size=1.4mm]
    \tikzstyle{grid} = [gray!50, very thin, step=.5cm]
    \begin{tikzpicture} [auto]
    \node (eta) at (-6, 0) [dot, label=below left:$\alternativeParameter$] {};
    \node (xi) at (0,0) [dot, label=below left:$\xi$]  {}; 
    \node (p) at (0,4) [dot, label=below left:$p$]  {}; 
    
    \draw [grid] ($(xi)+(-1,-1)$) grid ($(xi)+(1,1)$); 
    \draw [grid] ($(eta)+(-1,-1)$) grid ($(eta)+(1,1)$);
    \draw [bend right, thick, gray] (p) to  +(-1.5,0);
    \draw [bend right, thick, gray] (p) to  +(1.5,0);
    \node at ($(eta) + (1.4,1.2)$) {$\alternativeParameterSpace$};
    \node at ($(xi) + (1.4,1.2)$) {$\Xi$};
    \node at ($(p)+(2,0)$) {$\mathcal{M}$};
    
    \draw [->, black!80, very thin, bend right=15] ($(eta) + (1.5,0)$) to node {$f$} ($(xi) + (-1.5,0)$);
    \draw [->, black!80, very thin, bend left=30] ($(eta) + (0,1.5)$) to node {$\alternativeParametrisation$} ($(p) + (-2.2,0)$);
    \draw [->, black!80, very thin, bend right=0 ] ($(xi) + (0,1.5)$) to node [swap] {$\phi$} ($(p) + (0,-1)$);
    
    \draw [->] (xi) -- +(.95,-.20) node [anchor=west, inner sep = 3, node font=\small] {$\natpgrad$};
    \draw [->] (eta) -- +(.95,.45) node [anchor=west, inner sep = 3, node font=\small] {$\natapgrad$};
    \draw [->] (xi) -- +(.95,.65) node [anchor=west, inner sep = 3, node font=\small] {$df_{\alternativeParameter} \natapgrad$};
    \draw [->] (p) -- ($(p)+(330:1.4)$) node [anchor=north west, inner sep = 0, node font=\small, ] {\begin{tabular}{l}
         $d\phi_\xi \natpgrad =$ \\
         $d\phi_\xi df_{\alternativeParameter} \natapgrad$        
    \end{tabular}};
    \end{tikzpicture}  
    } 
    \caption{Two parametrisations of $\cM$ with different gradient vectors on the parameter space}
    \label{twoparams}
  \end{figure}

We will now make the above explicit. Let us choose the MP inverse as generalised inverse for $\natpgrad$ and consider an alternative parametrisation  $\alternativeParametrisation = \phi \circ f$ for a diffeomorphism $f\colon \alternativeParameterSpace \to \Xi$ (see Figure \ref{twoparams}). We denote the matrix of partial derivatives of $f$ at ${\alternativeParameter}$ with $F_i^j ({\alternativeParameter}) = \der[f^j]{{\alternativeParameter}^i}({\alternativeParameter})$. For $\xi = f({\alternativeParameter})$ we get the following relations:
\begin{align}
    \alternativePartial_i({\alternativeParameter}) &= F_{i}^j({\alternativeParameter}) \, \partial_j(\xi)\\
    \alternativeNabla_{\alternativeParameter} \cL &= F({\alternativeParameter}) \, \nabla_\xi \cL \\
    \alternativeG({\alternativeParameter}) &= F({\alternativeParameter}) \, G(\xi) \, F^T({\alternativeParameter}).
\end{align}
We map $\natapgrad$ to $T_\xi \Xi$ through $df_{\alternativeParameter}$ and get:
\begin{align}
    df_{\alternativeParameter} \natapgrad &=df_{\alternativeParameter}\left(\left(\alternativeG^+({\alternativeParameter})\alternativeNabla_{\alternativeParameter} \cL \right)^i\left.\frac{\partial}{\partial {\alternativeParameter}^i}\right|_{\alternativeParameter}\right) \\
    &= \left( \left(F({\alternativeParameter})G(\xi)F^T({\alternativeParameter})\right)^+ F({\alternativeParameter}) \, \nabla_\xi \cL \right)^i F_{i}^j({\alternativeParameter}) \left.\der{\xi^j}\right|_\xi \\
    &= \left( F^T({\alternativeParameter}) \left(F({\alternativeParameter})G(\xi)F^T({\alternativeParameter})\right)^+ F({\alternativeParameter}) \, \nabla_\xi \cL \right)^j \left.\der{\xi^j}\right|_\xi. \label{coef1}
\end{align}
We will write $y_\Xi, y_{\alternativeParameterSpace}$ for the coefficients of $\natpgrad$ and $df_{\alternativeParameter} \natapgrad$ respectively. From Theorem \ref{theorem1} and the fact that $F(\theta)$ is of full rank we know that $F(\theta) \, \nabla_\xi \cL$ lies in the image of $F(\theta)G(\xi)F^T(\theta)$. Therefore, by the definition of the MP inverse, we have that $\left(F(\theta)G(\xi)F^T(\theta)\right)^+ F(\theta) \, \nabla_\xi \cL = \argmin_x \{ ||x|| : F(\theta)G(\xi)F^T(\theta) x = F(\theta) \, \nabla_\xi \cL\}$, where $||\cdot||$ is the Euclidean norm on $\bR^d$.  The coefficients in \eqref{coef1} become:
\begin{align}
    y_{\alternativeParameterSpace} &= F^T({\alternativeParameter}) \left(F({\alternativeParameter})G(\xi)F^T({\alternativeParameter})\right)^+ F({\alternativeParameter}) \, \nabla_\xi \cL \\
    &= \argmin_y\left\{ ||\left(F^T({\alternativeParameter})\right)^{-1} y|| : G(\xi) y = \nabla_\xi \cL \right\} \label{yH},
\end{align}
where we substitute $y=F^T(\theta) x$ in the last line. 
\begin{remark}
Note that $||\left(F^T\right)^{-1} (\cdot)||$ is the pushforward of the norm on $\alternativeParameterSpace$ through $f$. This shows nicely the equivalence of the gradient for on the one hand constructing a different parametrisation $(\alternativeParametrisation)$, and on the other hand defining a different inner product $\left(||\left(F^T\right)^{-1} (\cdot)||\right)$ for the existing parametrisation $(\phi)$.
\end{remark}
Comparing the result to the natural parameter gradient on $\Xi$ gives:
\begin{align}
    \natpgrad &= \left( G^+(\xi) \nabla_\xi \cL\right)^i \left.\der{\xi^i}\right|_\xi\\
    &= (y_\Xi)^i \left.\der{\xi^i}\right|_\xi \\
    y_\Xi &= \argmin_y\{ ||y|| : G(\xi) y = \nabla_\xi \cL \} \label{yXi}.
\end{align}
Because the norms in \eqref{yH} and \eqref{yXi} are different, generally $y_{\alternativeParameterSpace} \neq y_\Xi$. However, both satisfy $G(\xi)y = \nabla_\xi \cL$ and therefore $G(\xi)(y_{\alternativeParameterSpace} - y_\Xi) = 0$. This implies:
\begin{align}
    d\phi_\xi \left(df_{\alternativeParameter} \natapgrad - \natpgrad \right) &=  \left(y_{\alternativeParameterSpace} - y_\Xi\right)^i \partial_i(\xi)\label{graddiff}\\ 
    &= 0,
\end{align}
where the last equality can be verified by taking the norm on the RHS of \eqref{graddiff} using that it is non-degenerate, like so: 
\begin{equation}
    ||\left(y_{\alternativeParameterSpace} - y_\Xi\right)^i \partial_i(\xi)||^2_g = \left(y_{\alternativeParameterSpace} - y_\Xi\right)^T G(\xi) \left(y_{\alternativeParameterSpace} - y_\Xi\right)=0.
\end{equation}
This shows that for overparametrised systems the natural parameter gradient is not reparametrisation invariant on the parameter space. However, as implied by Theorem \ref{theorem1}, the dependency on the parametrisation disappears when the gradient is mapped to the model. See Appendix \ref{exampleCalculationParameterSpace} for a worked-out example of the above discussion.

\section{Practical considerations for the natural gradient method}

The natural gradient method is performed by updating the current parameter vector $\xi \in \Xi$ in the direction of the vector $G^+(\xi) \nabla_\xi \cL \in \bR^d$. In case of a constrained parameter space, i.e. $\Xi$ is not the full space $\bR^d$, such as the space of covariance matrices, one runs the risk of stepping outside the parameter space, see also \cite{arjovsky2017}. This is called a constraint violation. One can use backprojection \cite{kushner1978}, addition of a penalty, and weight clipping \cite{gulrajani2017} to avoid these violations. Note that for a variety of neural network applications, including many supervised learning tasks, the parameter space is unconstrained. For these models however, the generalised inverse is often hard to compute due to the high number of parameters. In this context, the Woodbury matrix identity with damping is often used instead \cite{singh2020}. Investigating these topics further falls outside the scope of this paper. 

\subsubsection*{Reparametrisation (in)variance of the natural gradient method trajectory}
    We have shown in Theorem \ref{cor1} that from the perspective of the model, the natural parameter gradient is reparametrisation invariant. That is, for two parametrisations $\phi, \alternativeParametrisation$ for which $\alternativeParametrisation = \phi \circ f$ for a diffeomorphism $f$ and $\alternativeParameter = f^{-1}(\xi)$, we have that, 
    \begin{equation} \label{natgradmeth}
        \natpmgrad = \left(G^+(\xi) \nabla_\xi \cL \right)^i \partial_i = \left(\alternativeG^+(\alternativeParameter) \alternativeNabla_{\alternativeParameter} \cL \right)^j \alternativePartial_j = \natapmgrad.
    \end{equation}
    The natural gradient method is performed by updating the current parameter vector $\xi \in \bR^d$ in the direction of the vector $G^+(\xi) \nabla_\xi \cL \in \bR^d$. Equation \eqref{natgradmeth} implies that if we would update the parameters for both parametrisations an infinitesimal amount, this would give us the same result on the model. We would like to emphasise however that updating the parameters by a finite amount will in general result in different locations on the model. Therefore the natural gradient method trajectory is dependent on the choice of parametrisation. This is however not an issue specific to overparametrised models but with the natural gradient method in general. See Section 12 of \cite{martens2020new} for exact bounds on the invariance. \\
    
        
        

    \subsubsection*{Occurrence of non-proper points}

    We saw in the proof of Theorem \ref{cor1} that when $\vspan\{\partial_i(\xi)\}_i = \vspan\{\alternativePartial_j(\alternativeParameter)\}_j$ for two parametrisations $\phi$ and $\alternativeParametrisation$ with $\phi(\xi) = \alternativeParametrisation(\alternativeParameter)$, we have $\natpmgrad = \natapmgrad$. For $\phi(\xi) \in \Smooth(\cM)$ note that this equality holds in particular when $\vspan\{\partial_i(\xi)\}_i = \vspan\{\alternativePartial_j(\alternativeParameter)\}_j = T_p\cM$, i.e. $\phi$ is proper in $\xi$. Therefore we will now study when this is the case. We start by recalling some basic facts from smooth manifold theory: Let $M, N$ be smooth manifolds and $F:M \to N$ a smooth map. We call a point $p \in M$ a \textit{regular point} if $dF_p: T_p M \to T_{F(p)} N$ is surjective and a \textit{critical point} otherwise. A point $q \in N$ is called a \textit{regular value} if all the elements in $F^{-1}(q)$ are regular points, and a \textit{critical value} otherwise. If $M$ is $n$-dimensional, we say that a subset $S \subset M$ has measure zero in $M$, if for every smooth chart $(U, \alternativeParametrisation)$ for $M$, the subset $\alternativeParametrisation(S \cap U) \subset \mathbb{R}^{n}$ has $n$-dimensional measure zero. That is: $\forall \delta >0$, there exists a countable cover of $\alternativeParametrisation(S \cap U)$ consisting of open rectangles, the sum of whose volumes is less than $\delta$. 
    We have the following result based on Sard's theorem: 
    \begin{proposition}
        If $\Smooth(\cM)$ is a manifold, then the image of the set of points for which $\phi$ is not proper has measure zero in $\Smooth(\cM)$.  
    \end{proposition}
    \begin{proof}
        From the definition of $\Smooth(\cM)$ we know that for every $p \in \Smooth(\cM)$ there exists a $U_p$ open in $\cZ$ such that $U_p \cap \cM$ is an embedded submanifold of $\cZ$. Let $U \coloneqq \bigcup_{p\in \Smooth(\cM)} U_p$. Note that: $U \cap \cM = \Smooth(\cM)$ and therefore $\phi^{-1}(U) = \phi^{-1}(\Smooth(\cM))$. Since $U$ is open in $\cZ$, $\phi^{-1}(\Smooth(\cM))$ is an open subset of $\Xi$ and thus an embedded submanifold. Therefore we can consider the map:
        \begin{equation}
            \phi|_{\phi^{-1}(\Smooth(\cM))}: \phi^{-1}(\Smooth(\cM)) \to \Smooth(\cM)
        \end{equation}
        and note that the image of the set of points for which $\phi$ is not proper is equal to the set of critical values of $\phi|_{\phi^{-1}(\Smooth(\cM))}$ in $\Smooth(\cM)$. A simple application of Sard's theorem gives the result.
    \end{proof}
    This proposition implies that when $\Smooth(\cM)$ is a manifold, the set of points for which the pushforward of the natural parameter gradient is unequal to the natural gradient has measure zero in $\Smooth(\cM)$.

\section{Conclusion}
In this paper we have studied the natural parameter gradient, which was defined as the update direction of the natural gradient method, and its pushforward to the model in an overparametrised setting. We have seen that the latter is equal to the natural gradient under certain conditions. Furthermore we have proposed different notions of invariance and studied whether the natural parameter gradient satisfies these. From the perspective of the model, we have seen that the natural parameter gradient is reparametrisation invariant but that it is not parametrisation independent. Additionally, we saw that the natural parameter gradient is not reparametrisation invariant on the parameter space. We have argued, however, that this notion is less suitable in an overparametrised setting since multiple vectors on the parameter space can correspond to the same vector on the model. Finally we have given some practical considerations for the natural gradient method.

\section*{Acknowledgements}
JvO and NA acknowledge the support of the Deutsche Forschungsgemeinschaft Priority Programme “The Active Self” (SPP 2134). JM acknowledges support by the ERC under the European Union’s Horizon 2020 research and innovation programme (grant agreement no 757983), by the International Max Planck Research School for Mathematics in the Sciences and the Evangelisches Studienwerk Villigst e.V..

\printbibliography

\appendix

\section{Appendix}

\subsection{On the limitation of Definition \ref{parind}} \label{def4lim}

\begin{proposition}\label{parindprop}
    No non-trivial representation of a vector can be parametrisation independent in the sense of Definition \ref{parind}. More precisely, any representation satisfying Definition \ref{parind} has the property that for all parametrisation-coordinate pairs $\phi, \xi$ the following holds:
    \begin{equation}
        d\phi_\xi \obj(\phi,\xi) = 0.
    \end{equation}
\end{proposition}
\begin{proof}
    Let $\cM$ be a model and assume that $\obj$ is a representation satisfying Definition \ref{parind}. Let $\phi$ be a parametrisation of $\cM$ and $\xi^* \in \Xi$ be a fixed (arbitrary) element on the domain of $\phi$. Now consider the following function:
    \begin{align}
        f\colon \bR^d &\to \bR^d\\
        \theta = \left(\theta^1, .., \theta^d\right) &\mapsto \left(\left(\theta^1\right)^3, ..., \left(\theta^d\right)^3\right) + \xi^*.
    \end{align} 
    
    We define $\Theta = f^{-1}(\Xi)$ and $\psi = \phi \circ f|_{\Theta}$. 
    First note that, since $f$ is continuous, $\Theta$ is an open set. 
    Secondly, since $f$ is surjective, we have $\psi(\Theta) = \cM$. Therefore $\psi$ is a parametrisation of $\cM$. 
    It is easy to see that the differential of $f$ at $\theta = 0$,  $df_0$, is equal to zero and therefore by the chain rule we have $d\psi_0 = d(\phi \circ f|_{\Theta})_0 = d\phi_{f(0)} \circ df_0 = 0$. 
    Furthermore we have: $\psi(0) = \phi(\xi^*)$. Therefore in order for $\obj$ to satisfy equation \eqref{inveq}  we need that: 
    \begin{equation}
        d\phi_\xi^* \obj(\phi,\xi^*) = d\psi_0 \obj(\psi,0) = 0.
    \end{equation}
    Since $\xi^*$ was chosen arbitrarily, this implies that $\obj$ is a trivial representation. 
\end{proof}

\begin{remark}
    The function $f$ used in the proof above is actually a homeomorphism since it is a continuous bijection and its inverse, given by $f^{-1}(\xi) = \left(\sqrt[3]{\xi^1 - \left(\xi^*\right)^1},..., \sqrt[3]{\xi^d - \left(\xi^*\right)^d} \right)$, is also continuous. The inverse is however not differentiable and therefore $f$ is not a diffeomorphism, which is required for Definition \ref{invrep}.
\end{remark}

\subsection{Example calculation of parametrisation dependence on the model} \label{exampleCalculationModel}
We illustrate Example 2 in Section \ref{invman} with a specific calculation. Let us consider the following parametrisation:
\begin{align}
    \phi\colon (-1/8\pi, 5/8\pi) &\to (\bR^2, \bar{g}) \\
    \xi &\mapsto (x,y) = (\sin(2\xi), \sin(\xi)).
\end{align}
This gives $\xi_1 = 0, \xi_2 = \frac12 \pi$ in the above discussion. We get the following calculation for $\natpmgrada{0}$: 
\begin{align}
    \partial(\xi) &= d\phi_\xi \left(\left.\der{\xi}\right|_\xi\right) \\
    &= 2\cos(2\xi) \left.\der{x}\right|_{\phi(\xi)} + \cos(\xi) \left.\der{y}\right|_{\phi(\xi)}\\
    \partial(\xi_1) &= 2 \left.\der{x}\right|_{(0,0)} + \left.\der{y}\right|_{(0,0)}\\
    G(\xi_1) &= 2^2 + 1^1 = 5\\
    \natpmgrada{\xi_1} &=  \frac15 \left.\der{\xi}\right|_{\xi=0} \bigg(\cL(\sin(2\xi), \sin(\xi))\bigg) \left(2 \left.\der{x}\right|_{(0,0)} + \left.\der{y}\right|_{(0,0)} \right). \label{semigrad1}
\end{align}
Now let:
\begin{align}
    f\colon (-1/8\pi, 5/8\pi) &\to (-1/8\pi, 5/8\pi)\\
    \theta &\mapsto -(\theta-\frac14\pi). 
\end{align} 
This implies that $\theta_1 = f^{-1}(\xi_2) = 0$. We define the alternative parametrisation $\alternativeParametrisation = \phi \circ f$. Note that we have $\alternativeParametrisation(\theta) = (-\sin(2\theta), \sin(\theta))$ and thus $\alternativeParametrisation(\theta_1) = \phi(\xi_1) = (0,0)$. A similar calculation as before gives: 
\begin{align}
    \alternativePartial(\theta_1) &= -2 \left.\der{x}\right|_{(0,0)} + \left.\der{y}\right|_{(0,0)}\\
    \natapmgrada{\theta_1} &=  \frac15 \left.\der{\theta}\right|_{\theta=0} \bigg(\cL(-\sin(2\theta), \sin(\theta))\bigg) \left(-2 \left.\der{x}\right|_{(0,0)} + \left.\der{y}\right|_{(0,0)} \right). \label{semigrad2}
\end{align}
Note that because $\partial(\xi_1) \neq \alternativePartial(\theta_1)$, \eqref{semigrad1} and \eqref{semigrad2} are not equal to each other. We can therefore conclude that the natural parameter gradient is not parametrisation independent.

\subsection{Example calculation of reparametrisation (in)variance on the parameter space} \label{exampleCalculationParameterSpace}

We illustrate the discussion in Section \ref{paraminvparamspacessec} with a specific calculation. Let us consider the following setting:
\begin{align}
    \cZ &= \Xi = \alternativeParameterSpace = \bR^2 \\
    \phi(\xi_1, \xi_2) &= (\xi_1 + \xi_2, 0)\\
    f({\alternativeParameter}_1, {\alternativeParameter}_2) &= (2{\alternativeParameter}_1, {\alternativeParameter}_2)\\
    \cL(x,y) &= x^2
\end{align}
Plugging this into the expressions derived above gives: 
\begin{align}
    \alternativeParametrisation({\alternativeParameter}_1, {\alternativeParameter}_2) &= (2{\alternativeParameter}_1 + {\alternativeParameter}_2, 0)\\
    \partial_1(\xi) &= \partial_2(\xi) = \left.\der{x}\right|_{\phi(\xi)} &F({\alternativeParameter}) = \begin{bmatrix}
        2 & 0 \\
        0 & 1
        \end{bmatrix}\\
    \alternativePartial_1({\alternativeParameter}) &= 2\left.\der{x}\right|_{\alternativeParametrisation({\alternativeParameter})} &G(\xi) = \begin{bmatrix}
        1 & 1 \\
        1 & 1
        \end{bmatrix}\\
    \alternativePartial_2({\alternativeParameter}) &= \left.\der{x}\right|_{\alternativeParametrisation({\alternativeParameter})} &\alternativeG({\alternativeParameter}) = \begin{bmatrix}
        4 & 2 \\
        2 & 1
        \end{bmatrix}\\
    \nabla_\xi \cL &= (2(\xi_1 + \xi_2), 2(\xi_1 + \xi_2)) \\
    \alternativeNabla_{\alternativeParameter} \cL &= (4(2{\alternativeParameter}_1 + {\alternativeParameter}_2), 2(2{\alternativeParameter}_1 + {\alternativeParameter}_2)). 
\end{align}
Now we fix ${\alternativeParameter} = (1,1)$ and $\xi = f({\alternativeParameter}) = (2,1)$. We start by computing $y_\Xi$. From the above we know that:
\begin{equation}
    y_\Xi = \argmin_y\{ ||y|| : G(\xi) y = \nabla_\xi \cL \}.
\end{equation}
It can be easily verified that this gives $y_\Xi = (3,3)$. For $y_{\alternativeParameterSpace}$ we get: 
\begin{equation}
    y_{\alternativeParameterSpace} = \argmin_y\{ ||\left( F^T(\theta)\right)^{-1} y|| : G(\xi) y = \nabla_\xi \cL \},
\end{equation}
which gives: $y_{\alternativeParameterSpace} = (4 \frac45, 1\frac15)$. Evidently we have $y_\Xi \neq y_{\alternativeParameterSpace}$. Note however that when we map the difference of the two gradient vectors from $T_{(2,1)} \Xi$ to $T_{(3,0)}\cM$ through $d\phi_{(2,1)}$ we get: 
\begin{align}
    d\phi_\xi \left(df_{\alternativeParameter} \natapgrad - \natpgrad\right) &=  \left(y_{\alternativeParameterSpace} - y_\Xi\right)^i \partial_i(\xi)\\
    &= (4\frac45 - 3) \der{x}|_{(3,0)} + (1\frac15 - 3) \der{x}|_{(3,0)} \\
    &= 0.
\end{align}
This shows that the natural parameter gradient is in general not reparametrisation invariant on the parameter space, but that the dependency on the parametrisation disappears when the gradient is mapped to the model.

\end{document}